\documentclass{article}

\usepackage[preprint, nonatbib]{neurips_2019}

\usepackage[utf8]{inputenc} % allow utf-8 input
\usepackage[T1]{fontenc}    % use 8-bit T1 fonts
\usepackage{hyperref}       % hyperlinks
\usepackage{url}            % simple URL typesetting
\usepackage{booktabs}       % professional-quality tables
\usepackage{amsfonts}       % blackboard math symbols
\usepackage{nicefrac}       % compact symbols for 1/2, etc.
\usepackage{microtype}      % microtypography
\usepackage{graphicx}
\usepackage{subcaption}
\usepackage{dsfont}
\usepackage{amsthm}
\newtheorem{theorem}{Theorem}
\newtheorem{corollary}{Corollary}[theorem]
\newtheorem{lemma}{Lemma}

\usepackage[natbibapa]{apacite}
\usepackage{mathtools}
\usepackage{amsmath,amssymb}
\DeclareMathOperator{\E}{\mathbb{E}}
\DeclareMathOperator{\Var}{\mathbb{V}}
\let\P\relax
\DeclareMathOperator{\P}{\mathbb{P}}
\DeclareMathOperator{\defeq}{\dot{=}}
\DeclareMathOperator{\cov}{\mathbb{C}}

\title{Variance Reduced Advantage Estimation\\ with $\delta$ Hindsight Credit Assignment}

\author{
Kenny Young\\
Department of Computing Science\\
University of Alberta\\
Edmonton, AB, Canada\\
\texttt{kjyoung@ualberta.ca} \\
}

\begin{document}
\maketitle
\begin{abstract}
    Hindsight Credit Assignment (HCA) refers to a recently proposed family of methods for producing more efficient credit assignment in reinforcement learning. These methods work by explicitly estimating the probability that certain actions were taken in the past given present information. Prior work has studied the properties of such methods and demonstrated their behaviour empirically. We extend this work by introducing a particular HCA algorithm which has provably lower variance than the conventional Monte-Carlo estimator when the necessary functions can be estimated exactly. This result provides a strong theoretical basis for how HCA could be broadly useful.
\end{abstract}
\section{Advantage Estimation}
The advantage of an action $a$ in a state $s$ under policy $\pi$ is defined as
\begin{equation}
    \mathcal{A}_{\pi}(a,s)=Q_{\pi}(a,s)-V_{\pi}(s).
\end{equation}
Where $Q_{\pi}(a,s)$ is the action value and $V_{\pi}(s)$ is the state value under the policy $\pi$. 

Consider the problem of estimating the advantage of each action in a state from experience. One possible class of estimators is the $N$-step Monte-Carlo (MC) advantage estimators which can be defined as follows for a given $N$:
\begin{equation}\label{N_step_advantage_1}
    \hat{\mathcal{A}}^{MC}_{t,N}(a)=\frac{\mathds{1}(A_t=a)}{\pi(a|S_t)}\left(\sum_{k=1}^N \gamma^{k-1}R_{t+k}+\gamma^{N}\hat{V}(S_{t+N})-\hat{V}(S_t)\right).
\end{equation}

 In the case where $\hat{V}(s)$ is estimated perfectly, that is $\hat{V}(s)=V_\pi(s)\ \forall s$, we have that
\begin{equation*}
A_{\pi}(a,s)=\E_{\tau\sim\pi}[\hat{\mathcal{A}}^{MC}_{t,N}(a)|S_t=s].
\end{equation*}

Here we have used the notation $\E_{\tau\sim\pi}$ to mean the expectation over trajectories $\tau=(S_0,A_0,,R_{1},S_{1},...)$ sampled while following policy $\pi$. We will continue to use this notation throughout. We can also rewrite Equation~\ref{N_step_advantage_1} in terms of the TD error $\delta_t=R_{t+1}+\gamma\hat{V}(S_{t+1})-\hat{V}(S_t)$ as follows:
\begin{equation}\label{MC}
    \hat{\mathcal{A}}^{MC}_{t,N}(a)=\frac{\mathds{1}(A_t=a)}{\pi(a|S_t)}\sum_{k=0}^{N-1} \gamma^{k}\delta_{t+k}.
\end{equation}
One can define an $N$-step actor-critic algorithm using this advantage estimator. Larger $N$ will generally lead to a lower bias, but higher variance estimator.  Generalized advantage estimation~\citep{schulman2015high, AC_lambda} can be seen as using a weighted combination of these $N$-step estimators, which also gives rise to a bias-variance trade off based on the setting of the $\lambda$ hyper-parameter of the algorithm.

\section{Hindsight Credit Assignment}
An alternative approach to advantage estimation, termed Hindsight Credit Assignment (HCA) has been proposed by~\cite{HCA}. A similar approach was also proposed in concurrent work by~\cite{IAAE}\footnote{Preprint available via OpenReview.} under the name Independence-aware Advantage Estimation (IAAE). HCA attempts to make use of additional information about the long-term impact of actions on the state distribution. More precisely, HCA uses a learned estimate of the quantity ${\P_{\tau\sim\pi}(A_t=a|S_t=s,S_{t+k}=s^\prime)}$ for each $k$ from $1$ to $N$. That is, the probability that the action $a$ was selected at time $t$, given that the agent was in state $s$ at time $t$ and arrived in $s^{\prime}$ $k$ steps later. To keep the notation more compact we will use the following definition from now on:
\begin{equation*}
    p_k(a|s,s^\prime)\defeq \P_{\tau\sim\pi}(A_t=a|S_t=s,S_{t+k}=s^\prime).
\end{equation*}
Note that this depends on the policy $\pi$, but our notation has suppressed this dependence as it will not be necessary to distinguish between multiple policies in this report. Using $p_k(a|s,s^\prime)$, \cite{HCA} propose an advantage estimator of the following form:
\begin{multlined}[b]\hat{\mathcal{A}}^{HCA}_{t}(a)=r(S_t,a)-r(S_t)+\sum_{k\geq1} \gamma^{k}\left(\frac{p_k(a|S_t,S_{t+k})}{\pi(a|S_t)}-1\right)R_{t+k}\end{multlined}.

Where $r(s,a)=\E[R_{t+1}|S_t=s,A_t=a]$ and $r(s)=\E[R_{t+1}|S_t=s]$. They show that just like $\hat{\mathcal{A}}^{MC}_{t,N}(a)$, $\hat{\mathcal{A}}^{HCA}_{t}(a)$ is an unbiased estimator of the advantage, in the sense that:
\begin{equation*}
A_{\pi}(a,s)=\E_{\tau\sim\pi}[\hat{\mathcal{A}}^{HCA}_{t}(a)|S_t=s].
\end{equation*}
In practice, $r(S_t)$ and $p_k(a|S_t,S_{t+k})$ are replaced with learned approximations\footnote{\citet{HCA} treat the reward as deterministic function of $s$ and $a$, otherwise $r(s,a)$ would have to be either approximated as well, or sampled}, in which case the estimator is no longer guaranteed to be unbiased. 

The idea of the HCA estimator is that rather than estimating the advantage of an action using only sampled trajectories for when the action was actually taken, we can use all trajectories, where each state we visit is explicitly weighted by the probability of arriving there if we had selected action $a$ at time $t$. One advantage of this is that if the action choice has no impact on the state $k$ steps later, we will assign equal credit to all actions instead of potentially adding variance by crediting the selected action despite the fact that it wouldn't have changed anything that far into the future. 

As an example where HCA could make a big difference, consider the ATARI game space invaders, where the player can move along the bottom of the screen as well as shoot lasers upward. Shooting a laser has a lasting impact in that the laser persists on the screen for several steps and may ultimately destroy an alien. By contrast, jittering randomly on the bottom of the screen has little long-term impact, and there may be many such action sequences that lead to the same state in a few steps, thus it might be wasteful to credit those actions far into the future. 

\section{$\delta$ Hindsight Credit Assignment}
\citet{HCA}, provide several motivations for HCA, including the potential for reduced variance. This is based on an intuitive argument that crediting a particular selected action $a$ for reaching a later state $s$ when any other action $a^\prime$ would have been just as likely to reach $s$ overemphasizes $a$ and is a source of variance. This is not however formalized, and in fact there are cases where HCA results in a higher variance advantage estimator than the conventional MC estimator, see Figure~\ref{HCA_CE} for a simple example. The authors of IAAE~\citep{IAAE} acknowledge this possibility and propose a method which combines MC and HCA style estimators using a learned parameter that distributes each reward between the two estimators with the aim of minimizing a lower bound on the variance.

\begin{figure}[!ht]
\centering
\includegraphics[width=\textwidth]{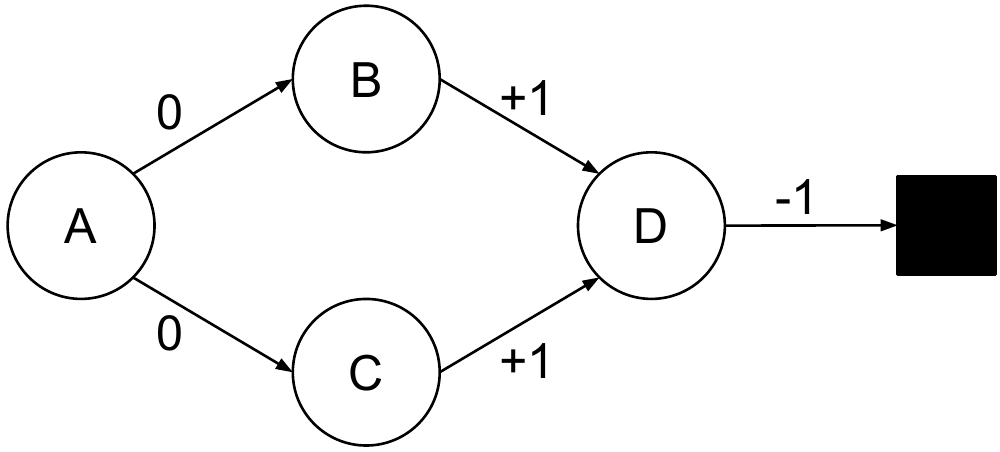}
\caption{An example where the HCA advantage estimator has higher variance than the conventional MC estimator. Circles represent states, while arrows represent transitions, and the black square represents the terminal state. The two paths available from A represent two possible actions, all other transitions are independent of the action selections. Numbers above transitions represent the corresponding rewards. Since the total reward along both paths is zero, the Monte-Carlo advantage estimator is always zero for both actions in state A, and thus has zero variance. On the other hand, assuming a uniform random policy, State D has the same probability of resulting from either action. Thus, the only non-zero contribution to the HCA estimator will be from state B or C. As a result, given $S_t=A$, $\hat{\mathcal{A}}^{HCA}_{t}(a)$ will be $1$ for the selected action and $-1$ for the other action, giving an overall variance of $1$ for each action. The key feature of this example that leads HCA to increase the variance is that the state information is useful in predicting the past action, but not the subsequent reward.}
\label{HCA_CE}
\end{figure}

In this report, we show that a simple change to the HCA approach yields an estimator that remains unbiased when the necessary functions can be estimated exactly, but always has equal or lower variance compared to the conventional MC estimator defined in Equation~\ref{MC}. While in practice the functions used in the estimator need to be approximated, this result provides a strong theoretical basis for how such alternative advantage estimators could be broadly useful.

The main difference between our estimator and the HCA estimator is that instead of rewards, our estimator decomposes the advantage as a sum of the TD-errors, $\delta_t$. Our advantage estimator, which we term the $\delta$HCA estimator is defined as follows:
\begin{equation}
    \hat{\mathcal{A}}^{\delta HCA}_{t,N}(a)=\frac{\mathds{1}(A_t=a)}{\pi(a|S_t)}\delta_{t}+\sum_{k=1}^{N-1} \gamma^{k}\frac{p_k(a|S_t,S_{t+k})}{\pi(a|S_t)}\delta_{t+k}.
    \label{delta-HCA}
\end{equation}
Our main result is made possible because, unlike rewards, TD-errors have expectation zero and are inherently uncorrelated when the value function is learned exactly. Note that the first term on the right side of Equation~\ref{delta-HCA} is identical to the first term in Equation~\ref{MC}, hence the estimators only differ for $N>1$. For every other term we have replaced the identity function $\mathds{1}(A_t=a)$ with the probability ${p_k(a|S_t,S_{t+k})}$. 

%TODO: work on this example
%This kind of structure also seems common in the real world, for example, in every day life waving your arms around in empty space has little lasting impact on the environment (provided no one is watching to judge you) and so could be safely ignored when applying credit for future rewards. On the other hand throwing your laptop through a nearby window will have a lasting impact and thus should perhaps be considered when assigning credit further into the future.

We begin by showing that $\hat{\mathcal{A}}^{\delta HCA}_{t,N}(a)$ has the same expectation as $\hat{\mathcal{A}}^{MC}_{t,N}(a)$. As a reminder, this is true only under the assumption of access to the ground truth ${p_k(a|s,s^\prime)}$, however this result does not require the ground truth $V_\pi$ and will in fact hold for any value function estimate $\hat{V}$.

\begin{theorem}
\begin{equation*}
\E_{\tau\sim\pi}[\hat{\mathcal{A}}^{\delta HCA}_{t,N}(a)|S_t=s]=\E_{\tau\sim\pi}[\hat{\mathcal{A}}^{MC}_{t,N}(a)|S_t=s]\ \forall s,a.
\end{equation*}
\end{theorem}
\begin{proof}
By linearity of expectation, we can rewrite $\E[\hat{\mathcal{A}}^{\delta HCA}_{t,N}(a)|S_t=s]$ as follows:
\begin{equation}\label{E_delta_HCA}
    \begin{multlined}[b]\E_{\tau\sim\pi}\left[\hat{\mathcal{A}}^{\delta HCA}_{t,N}(a)|S_t=s\right]=
    \E_{\tau\sim\pi}\left[\frac{\mathds{1}(A_t=a)}{\pi(a|S_t)}\delta_{t}\middle|S_t=s\right]\\
    +\sum_{k=1}^{N-1} \gamma^{k}\E_{\tau\sim\pi}\left[\frac{p_k(a|S_t,S_{t+k})}{\pi(a|S_t)}\delta_{t+k}\middle|S_t=s\right]
    \end{multlined}.
\end{equation}

The first expectation on the right is trivially the same as the first term in Equation~\ref{MC} so let's focus our attention on the expectations inside the summation sign. We begin by showing that we can use the Markov property to rewrite $p_k(a|S_t,S_{t+k})$ as an expectation additionally conditioned on $R_{t+k+1}$, $S_{t+k+1}$, which will allow us to bring $\delta_{t+k}$ inside this expectation:
\begin{align*}
    p_k(a|S_t,S_{t+k})&=\P_{\tau\sim\pi}(A_t=a|S_t,S_{t+k})\\
    &=\frac{\P_{\tau\sim\pi}( R_{t+k+1},S_{t+k+1}|S_{t+k})\P_{\tau\sim\pi}(A_t=a|S_t,S_{t+k})}{\P_{\tau\sim\pi}( R_{t+k+1},S_{t+k+1}|S_{t+k})}\\
    &\stackrel{(a)}{=}\frac{\P_{\tau\sim\pi}( R_{t+k+1},S_{t+k+1}|A_t,S_t,S_{t+k})\P_{\tau\sim\pi}(A_t=a|S_t,S_{t+k})}{\P_{\tau\sim\pi}( R_{t+k+1},S_{t+k+1}|S_t,S_{t+k})}\\
    &=\P_{\tau\sim\pi}(A_t=a|S_t,S_{t+k}, R_{t+k+1},S_{t+k+1})\\
    &=\E_{\tau\sim\pi}[\mathds{1}(A_t=a)|S_t,S_{t+k}, R_{t+k+1},S_{t+k+1}],\\
\end{align*}
Where $(a)$ makes use of the Markov property. Using this identity, can rewrite the terms in Equation~\ref{E_delta_HCA} as follows:
\begin{align*}
    &\E_{\tau\sim\pi}\left[\frac{p_k(a|S_t,S_{t+k})}{\pi(a|S_t)}\delta_{t+k}\middle|S_t=s\right]\\
    =&\E_{\tau\sim\pi}\left[\frac{\E_{\tau\sim\pi}[\mathds{1}(A_t=a)|S_t,S_{t+k}, R_{t+k+1},S_{t+k+1}]}{\pi(a|S_t)}\delta_{t+k}\middle|S_t=s\right]\\
    =&\E_{\tau\sim\pi}\left[\E_{\tau\sim\pi}\left[\frac{\mathds{1}(A_t=a)}{\pi(a|S_t)}\delta_{t+k}\middle|S_t,S_{t+k}, R_{t+k+1},S_{t+k+1}\right]\middle|S_t=s\right]\\
    \stackrel{(a)}{=}&\E_{\tau\sim\pi}\left[\frac{\mathds{1}(A_t=a)}{\pi(a|S_t)}\delta_{t+k}\middle|S_t=s\right].
\end{align*}
Where (a) uses the law of total expectation. Substituting this result into Equation~\ref{E_delta_HCA} gives:
\begin{align*}
\E_{\tau\sim\pi}\left[\hat{\mathcal{A}}^{\delta HCA}_{t,N}(a)\middle|S_t=s\right]&=\sum_{k=0}^{N-1} \gamma^{k}\E_{\tau\sim\pi}\left[\frac{\mathds{1}(A_t=a)}{\pi(a|S_t)}\delta_{t+k}\middle|S_t=s\right]\\
&=\E_{\tau\sim\pi}\left[\sum_{k=0}^{N-1} \gamma^{k}\frac{\mathds{1}(A_t=a)}{\pi(a|S_t)}\delta_{t+k}\middle|S_t=s\right]\\
&=\E_{\tau\sim\pi}\left[\frac{\mathds{1}(A_t=a)}{\pi(a|S_t)}\sum_{k=0}^{N-1} \gamma^{k}\delta_{t+k}\middle|S_t=s\right]\\
&=\E_{\tau\sim\pi}\left[\hat{\mathcal{A}}^{MC}_{t,N}(a)\middle|S_t=s\right].
\end{align*}
\end{proof}

\begin{corollary}
The bias of $\hat{\mathcal{A}}^{\delta HCA}_{t,N}(a)$ as an estimator of $\mathcal{A}_\pi(a,s)$ is equal to that of $\hat{\mathcal{A}}^{MC}_{t,N}(a)$ i.e.
\begin{equation*}
(\E_{\tau\sim\pi}[\hat{\mathcal{A}}^{\delta HCA}_{t,N}(a)|S_t=s]-\mathcal{A}_\pi(a,s))^2=(\E_{\tau\sim\pi}[\hat{\mathcal{A}}^{MC}_{t,N}(a)|S_t=s]-\mathcal{A}_\pi(a,s))^2.
\end{equation*}
\end{corollary}
This also implies that, as $N$ is increased, $\hat{\mathcal{A}}^{\delta HCA}_{t,N}(a)$ will see a reduction in bias under the same circumstances as $\hat{\mathcal{A}}^{MC}_{t,N}(a)$.

\section{The $\delta$ Hindsight Credit Assignment Estimator has Lower Variance than the Monte-Carlo Estimator}
We now present the main result of this report, a theorem formalizing that in the case where we have access to the exact value function and action probabilities, the $\delta$HCA estimator will always have equal or lower variance compared to the conventional MC advantage estimator. For convenience, we define two additional probabilities related to $p_k(a|s,s^\prime)$:
\begin{align*}
    &p_k(s^\prime|s)\defeq \P_{\tau\sim\pi}(S_{t+k}=s^\prime|S_t=s)\text{ and}\\
    &p_k(s^\prime|s,a)\defeq \P_{\tau\sim\pi}(S_{t+k}=s^\prime|S_t=s, A_t=a),
\end{align*}
where we have abused notation in favor of readability to overload $p_k$ based on the nature of its arguments. In what follows we will use $\Var(x)$ to denote the variance of a random variable $x$ and $\cov(x,y)$ to denote the covariance between two random variables.
\begin{theorem}\label{low_variance}
Given $\hat{V}(s)=V_\pi(s)\ \forall s$
\begin{equation*}
    \Var_{\tau\sim\pi}(\hat{\mathcal{A}}^{\delta HCA}_{t,N}(a)|S_t=s)\leq \Var_{\tau\sim\pi}(\hat{\mathcal{A}}^{MC}_{t,N}(a)|S_t=s)\ \forall s,a,
\end{equation*}
\end{theorem}
To prove Theorem~\ref{low_variance} it will be useful to first prove a lemma regarding the covariance between quantities at different times. This will allow us to break down the variance of each estimator as a sum of the variances of it's component terms. Toward this, it will be convenient to define notation for partial trajectories as follows:
\begin{equation*}
    \tau_{0:t}=(S_0,A_0,R_1,S_1,...,S_t).
\end{equation*}
Now the lemma:
\begin{lemma}\label{uncorrolated}
Given $\hat{V}(s)=V_\pi(s)\ \forall s$. Let $t^{\prime\prime}\geq t^\prime\geq t$ and $b^\prime(\tau_{0:t^\prime})$, $b^{\prime\prime}(\tau_{0:t^{\prime\prime}})$ be arbitrary functions depending only on the trajectory up to time $t^\prime$ and $t^{\prime\prime}$ respectively.
\begin{equation*}
    \cov_{\tau\sim\pi}\left(b^\prime(\tau_{0:t^\prime}),b^{\prime\prime}(\tau_{0:t^{\prime\prime}})\delta_{t^{\prime\prime}}\middle|S_t=s\right)=0,
\end{equation*}
\end{lemma}
\begin{proof}
\begin{align*}
&\cov_{\tau\sim\pi}\left(b^\prime(\tau_{0:t^\prime}),b^{\prime\prime}(\tau_{0:t^{\prime\prime}})\delta_{t^{\prime\prime}}\middle|S_t=s\right)\\
&\begin{multlined}[b]=\E_{\tau\sim\pi}[b^\prime(\tau_{0:t^\prime})b^{\prime\prime}(\tau_{0:t^{\prime\prime}})\delta_{t^{\prime\prime}}|S_t=s]\\
-\E_{\tau\sim\pi}[b^\prime(\tau_{0:t^\prime})|S_t=s]\E_{\tau\sim\pi}[b^{\prime\prime}(\tau_{0:t^{\prime\prime}})\delta_{t^{\prime\prime}}|S_t=s]
\end{multlined}.
\end{align*}
Let's begin by working out the first term:
\begin{align*}
    &\E_{\tau\sim\pi}[b^\prime(\tau_{0:t^\prime})b^{\prime\prime}(\tau_{0:t^{\prime\prime}})\delta_{t^{\prime\prime}}|S_t=s]\\
    &=\E_{\tau\sim\pi}[\E_{\tau\sim\pi}[b^\prime(\tau_{0:t^\prime})b^{\prime\prime}(\tau_{0:t^{\prime\prime}})\delta_{t^{\prime\prime}}|\tau_{0:t^{\prime\prime}}]|S_t=s]\\
    &=\E_{\tau\sim\pi}[b^\prime(\tau_{0:t^\prime})b^{\prime\prime}(\tau_{0:t^{\prime\prime}})\E_{\tau\sim\pi}[\delta_{t^{\prime\prime}}|\tau_{0:t^{\prime\prime}}]|S_t=s]\\
    &\stackrel{(a)}{=}\E_{\tau\sim\pi}[b^\prime(\tau_{0:t^\prime})b^{\prime\prime}(\tau_{0:t^{\prime\prime}})\E_{\tau\sim\pi}[\delta_{t^{\prime\prime}}|S_{t^{\prime\prime}}]|S_t=s]\\
    &\stackrel{(b)}{=}\E_{\tau\sim\pi}[b^\prime(\tau_{0:t^\prime})b^{\prime\prime}(\tau_{0:t^{\prime\prime}})0|S_t=s]\\
    &=0,
\end{align*}
where (a) follows from the Markov property, while (b) follows from the fact that the true value function obeys the Bellman equation. For the second term we can show by analogous reasoning that:
\begin{equation*}
    \E_{\tau\sim\pi}[b^{\prime\prime}(\tau_{0:t^{\prime\prime}})\delta_{t^{\prime\prime}}|S_t=s]=0.
\end{equation*}
Thus indeed:
\begin{equation*}
    \cov_{\tau\sim\pi}\left(b^\prime(\tau_{0:t^\prime}),b^{\prime\prime}(\tau_{0:t^{\prime\prime}})\delta_{t^{\prime\prime}}\middle|S_t=s\right)=0.
\end{equation*}
\end{proof}
Stated simply, the main idea of this lemma is that, when $\hat{V}(s)=V_\pi(s)\ \forall s$, $\delta_t$ is conditionally uncorrelated with everything that happens up to and including time $t$. That is, at time $t$ the expectation of $\delta_t$ is always zero regardless of the history up to that point. With this lemma in place we proceed with the proof of Theorem~\ref{low_variance}.
\begin{proof}[\textbf{Proof} (Theorem~\ref{low_variance}).]
We will proceed by writing the variance of each estimator in a form such that we can directly compare the two. Beginning with $\hat{\mathcal{A}}^{MC}_{t,N}(a)$:
\begin{align*}
    \Var_{\tau\sim\pi}(\hat{\mathcal{A}}^{MC}_{t,N}(a)|S_t=s)&=\Var_{\tau\sim\pi}\left(\frac{\mathds{1}(A_t=a)}{\pi(a|S_t)}\sum_{k=0}^{N-1} \gamma^{k}\delta_{t+k}\middle|S_t=s\right)\\
    &=\sum_{k=0}^{N-1}\sum_{j=0}^{N-1}\gamma^{(k+j)}\cov\left(\frac{\mathds{1}(A_t=a)}{\pi(a|S_t)}\delta_{t+k},\frac{\mathds{1}(A_t=a)}{\pi(a|S_t)}\delta_{t+j}\middle|S_t=s\right)\\
    &\begin{multlined}[b]=2\sum_{k=0}^{N-1}\sum_{j=k+1}^{N-1}\gamma^{(k+j)}\cov\left(\frac{\mathds{1}(A_t=a)}{\pi(a|S_t)}\delta_{t+k},\frac{\mathds{1}(A_t=a)}{\pi(a|S_t)}\delta_{t+j}\middle|S_t=s\right)\\
    +\sum_{k=0}^{N-1}\gamma^{2k}\Var_{\tau\sim\pi}\left(\frac{\mathds{1}(A_t=a)}{\pi(a|S_t)}\delta_{t+k}\middle|S_t=s\right)
    \end{multlined}.
\end{align*}.

% &\begin{multlined}[b]
%     =\Var_{\tau\sim\pi}\left(\frac{\mathds{1}(A_t=a)}{\pi(a|S_t)}\delta_{t}|S_t=s\right)\\+\sum_{k=1}^{N-1}\gamma^{2k}\Var_{\tau\sim\pi}\left(\frac{\mathds{1}(A_t=a)}{\pi(a|S_t)}\delta_{t+k}|S_t=s\right)
%     \end{multlined}.
A straightforward application of Lemma~\ref{uncorrolated} with $t^\prime=t+k+1$, $t^{\prime\prime}=t+j$, $b^\prime(\tau_{0:t^\prime})=\frac{\mathds{1}(A_t=a)}{\pi(a|S_t)}\delta_{t+k}$ and $b^{\prime\prime}(\tau_{0:t^{\prime\prime}})=\frac{\mathds{1}(A_t=a)}{\pi(a|S_t)}$ gives us the result that every term in the double summation is zero. Thus we get: 
\begin{equation*}
    \Var_{\tau\sim\pi}(\hat{\mathcal{A}}^{MC}_{t,N}(a)|S_t=s)=\sum_{k=0}^{N-1}\gamma^{2k}\Var_{\tau\sim\pi}\left(\frac{\mathds{1}(A_t=a)}{\pi(a|S_t)}\delta_{t+k}\middle|S_t=s\right).
\end{equation*}
Moving on to the expression for $\hat{\mathcal{A}}^{\delta HCA}_{t,N}(a)$
\begin{align*}
    \Var_{\tau\sim\pi}&(\hat{\mathcal{A}}^{\delta HCA}_{t,N}(a)|S_t=s)\\
    &=\Var_{\tau\sim\pi}\left(\frac{\mathds{1}(A_t=a)}{\pi(a|S_t)}\delta_{t}+\sum_{k=1}^{N-1} \gamma^{k}\frac{p_k(a|S_t,S_{t+k})}{\pi(a|S_t)}\delta_{t+k}\middle|S_t=s\right)\\
    &\begin{multlined}[b]=\sum_{k=1}^{N-1}\gamma^k\cov\left(\frac{\mathds{1}(A_t=a)}{\pi(a|S_t)}\delta_t,\frac{p_k(a|S_t,S_{t+k})}{\pi(a|S_t)}\delta_{t+k}\middle|S_t=s\right)\\
    +2\sum_{k=1}^{N-1}\sum_{j=k+1}^{N-1}\gamma^{k+j}\cov\left(\frac{p_k(a|S_t,S_{t+k})}{\pi(a|S_t)}\delta_{t+k},\frac{p_j(a|S_t,S_{t+j})}{\pi(a|S_t)}\delta_{t+j}\middle|S_t=s\right)\\
    +\Var_{\tau\sim\pi}\left(\frac{\mathds{1}(A_t=a)}{\pi(a|S_t)}\delta_{t}\middle|S_t=s\right)\\
    +\sum_{k=1}^N\gamma^{2k}\Var_{\tau\sim\pi}\left(\frac{p_k(a|S_t,S_{t+k})}{\pi(a|S_t)}\delta_{t+k}\middle|S_t=s\right)
    \end{multlined}.
\end{align*}
We can apply Lemma~\ref{uncorrolated} similarly to above to show that each of the first two lines in the above sum are zero leaving us with:
\begin{equation*}
    \begin{multlined}[b]
    \Var_{\tau\sim\pi}(\hat{\mathcal{A}}^{\delta HCA}_{t,N}(a)|S_t=s)=\Var_{\tau\sim\pi}\left(\frac{\mathds{1}(A_t=a)}{\pi(a|S_t)}\delta_{t}|S_t=s\right)\\
    +\sum_{k=1}^{N-1}\gamma^{2k}\Var_{\tau\sim\pi}\left(\frac{p_k(a|S_t,S_{t+k})}{\pi(a|S_t)}\delta_{t+k}\middle|S_t=s\right)
    \end{multlined}.
\end{equation*}
Notice that variances of $\hat{\mathcal{A}}^{\delta HCA}_{t,N}(a)$ and $\hat{\mathcal{A}}^{MC}_{t,N}(a)$ share the first term, so it suffices to show:
\begin{equation*}
    \sum_{k=1}^{N-1}\gamma^{2k}\Var_{\tau\sim\pi}\left(\frac{p_k(a|S_t,S_{t+k})}{\pi(a|S_t)}\delta_{t+k}\middle|S_t=s\right)\leq\sum_{k=1}^{N-1}\gamma^{2k}\Var_{\tau\sim\pi}\left(\frac{\mathds{1}(A_t=a)}{\pi(a|S_t)}\delta_{t+k}\middle|S_t=s\right).
\end{equation*}
In fact we can show that the inequality holds for each term independently such that:
\begin{equation}\label{var_inequality}
    \Var_{\tau\sim\pi}\left(\frac{p_k(a|S_t,S_{t+k})}{\pi(a|S_t)}\delta_{t+k}\middle|S_t=s\right)\leq \Var_{\tau\sim\pi}\left(\frac{\mathds{1}(A_t=a)}{\pi(a|S_t)}\delta_{t+k}\middle|S_t=s\right)\ \forall k>0.
\end{equation}
To see this, begin by expanding the right hand side:
\begin{align*}
    &\Var_{\tau\sim\pi}\left(\frac{\mathds{1}(A_t=a)}{\pi(a|S_t)}\delta_{t+k}\middle|S_t=s\right)\\
    =&\E_{\tau\sim\pi}\left[\left(\frac{\mathds{1}(A_t=a)}{\pi(a|S_t)}\delta_{t+k}\right)^2\middle|S_t=s\right]\\
    =&\pi(a|s)\sum\limits_{s^\prime}p_k(s^\prime|s,a)\E_{\tau\sim\pi}\left[\left(\frac{\delta_{t+k}}{\pi(a|S_t)}\right)^2\middle|S_t=s,S_{t+k}=s^\prime\right]\\
    =&\sum\limits_{s^\prime}\frac{p_k(s^\prime|s,a)}{\pi(a|s)}\E_{\tau\sim\pi}\left[\delta_{t+k}^2\middle|S_t=s,S_{t+k}=s^\prime\right].
\end{align*}

Now for the left hand side:
\begin{align*}
&\Var_{\tau\sim\pi}\left(\frac{p_k(a|S_t,S_{t+k})}{\pi(a|S_t)}\delta_{t+k}|S_t=s\right)\\
=&\E_{\tau\sim\pi}\left[\left(\frac{p_k(a|S_t,S_{t+k})}{\pi(a|S_t)}\delta_{t+k}\right)^2\middle|S_t=s\right]\\
=&\sum\limits_{s^\prime}p_k(s^\prime|s)\E_{\tau\sim\pi}\left[\left(\frac{p_k(a|S_t,S_{t+k})}{\pi(a|S_t)}\delta_{t+k}\right)^2\middle|S_t=s,S_{t+k}=s^\prime\right]\\
=&\sum\limits_{s^\prime}p_k(s^\prime|s)\left(\frac{p_k(a|s,s^\prime)}{\pi(a|s)}\right)^2\E_{\tau\sim\pi}\left[\delta_{t+k}^2\middle|S_t=s,S_{t+k}=s^\prime\right]\\
\stackrel{(a)}{=}&\sum\limits_{s^\prime}p_k(s^\prime|s)\frac{p_k(s^\prime|s,a)}{p_k(s^\prime|s)}\frac{p_k(a|s,s^\prime)}{\pi(a|s)}\E_{\tau\sim\pi}\left[\delta_{t+k}^2\middle|S_t=s,S_{t+k}=s^\prime\right]\\
=&\sum\limits_{s^\prime}\frac{p_k(s^\prime|s,a)}{\pi(a|s)}p_k(a|s,s^\prime)\E_{\tau\sim\pi}\left[\delta_{t+k}^2\middle|S_t=s,S_{t+k}=s^\prime\right],
\end{align*}
where (a) uses Bayes theorem to make the replacement:
\begin{equation*}
\frac{p_k(a|s,s^\prime)}{\pi(a|s)}=\frac{p_k(s^\prime|s,a)}{p_k(s^\prime|s)}.
\end{equation*}

Now substituting the two expressions into Equation~\ref{var_inequality} gives
\begin{equation*}
\begin{multlined}[b]
\sum\limits_{s^\prime}\frac{p_k(s^\prime|s,a)}{\pi(a|s)}p_k(a|s,s^\prime)\E_{\tau\sim\pi}\left[\delta_{t+k}^2\middle|S_t=s,S_{t+k}=s^\prime\right]\\\leq\sum\limits_{s^\prime}\frac{p_k(s^\prime|s,a)}{\pi(a|s)}\E_{\tau\sim\pi}\left[\delta_{t+k}^2\middle|S_t=s,S_{t+k}=s^\prime\right]
\end{multlined},
\end{equation*}
which indeed holds because the only difference is that the left hand side includes an additional $p_k(a|s,s^\prime)$ for each term in the sum. Since this probability is always less than one and each term is positive, this has the effect of reducing each term. Based on the previous analysis, this implies the intended result.
\end{proof}
Aside from proving the result, the final expression in the above proof sheds light on exactly how the variance of $\hat{\mathcal{A}}^{\delta HCA}_{t,N}(a)$ is reduced. Whenever the probability $p_k(a|S_t,S_{t+k})$ is low, the contribution of the associated $\delta_{t+k}$ is suppressed because the action being credited was not a significant factor in arriving at state $S_{t+k}$.

\section{Variance in the Policy Gradient Update}
We have established that the $\delta$HCA advantage estimator has lower variance than the conventional MC advantage estimator (at least under the assumption that both the action prediction model and value function are learned perfectly for the current policy). Perhaps more important is the impact on the variance in the policy gradient update. For some estimate $\hat{\mathcal{A}}_t$ of the advantage, the policy gradient update can be expressed as follows:
\begin{equation*}
    \theta_{t+1}\leftarrow\theta_t+\alpha \sum_a \hat{\mathcal{A}}_t(a)\nabla_\theta\pi(a|S_t),
\end{equation*}
where $\alpha$ is the step-size. Let's now take a deeper look at the variance of the quantity $\Delta_{\theta,t}=\sum_a \hat{\mathcal{A}}_t(a)\nabla_\theta\pi(a|S_t)$, for both the MC and $\delta$HCA advantage estimators. For a general advantage estimator:
\begin{align}
   \Var_{\tau\sim\pi}\left(\Delta_{\theta,t}\middle|S_t\right)&= \Var_{\tau\sim\pi}\left(\sum_a \hat{\mathcal{A}}_t(a)\nabla_\theta\pi(a|S_t)\middle|S_t\right)\nonumber\\
   &\begin{multlined}[b]=\sum_a \nabla_\theta\pi(a|S_t)^2\Var_{\tau\sim\pi}(\hat{\mathcal{A}}_t(a)|S_t)\\
   +2\sum_{a^\prime>a}\nabla_\theta\pi(a|S_t)\nabla_\theta\pi(a^\prime|S_t)\cov_{\tau\sim\pi}(\hat{\mathcal{A}}_t(a),\hat{\mathcal{A}}_t(a^\prime)|S_t)
   \end{multlined}.\label{update_variance}
\end{align}

To move further with this expression we need to understand not only the variance of individual advantage estimators, but also the covariance between advantage estimators for different actions. For the $N$-step MC estimator $\hat{\mathcal{A}}^{MC}_{t,N}(a)$ this covariance is very simple. In the MC case the advantage estimator is nonzero only for the action selected, hence the covariance between advantage estimators is simply the negated product of their expectation. Assuming $\hat{V}(s)=V_\pi(s)$, the expectation of the estimator is the true advantage, hence in this case we get:
\begin{equation*}
    \cov_{\tau\sim\pi}(\hat{\mathcal{A}}^{MC}_{t,N}(a),\hat{\mathcal{A}}^{MC}_{t,N}(a^\prime)|S_t)=-\mathcal{A}_{\pi}(a,S_t)\mathcal{A}_{\pi}(a^\prime,S_t).
\end{equation*}

For the HCA estimator the situation is more complicated.
\begin{align*}
    &\cov_{\tau\sim\pi}(\hat{\mathcal{A}^\prime}_{t,N}(a),\hat{\mathcal{A}^\prime}_{t,N}(a^\prime)|S_t)=\\
    &\sum_{k=1}^{N-1}\sum_{j=1}^{N-1}\gamma^{(k+j)}\cov_{\tau\sim\pi}\left(\frac{p_k(a|S_t,S_{t+k})}{\pi(a|S_t)}\delta_{t+k},\frac{p_j(a^\prime|S_t,S_{t+j})}{\pi(a^\prime|S_t)}\delta_{t+j}\middle|S_t\right)\\
    &+\sum_{j=1}^{N-1}\gamma^j\cov_{\tau\sim\pi}\left(\frac{\mathds{1}(A_t=a)}{\pi(a|S_t)}\delta_{t},\frac{p_j(a^\prime|S_t,S_{t+j})}{\pi(a^\prime|S_t)}\delta_{t+j}\middle|S_t\right)\\
    &+\sum_{k=1}^{N-1}\gamma^k\cov_{\tau\sim\pi}\left(\frac{\mathds{1}(A_t=a^\prime)}{\pi(a^\prime|S_t)}\delta_{t},\frac{p_k(a|S_t,S_{t+k})}{\pi(a|S_t)}\delta_{t+k}\middle|S_t\right)\\
    &+\cov_{\tau\sim\pi}\left(\frac{\mathds{1}(A_t=a^\prime)}{\pi(a^\prime|S_t)}\delta_{t},\frac{\mathds{1}(A_t=a)}{\pi(a|S_t)}\delta_{t}\middle|S_t\right).
\end{align*}
We can apply Lemma~\ref{uncorrolated} to conclude that, in the above sum, the second and third lines as well as the off diagonal terms of the first line are all zero. The fourth line is easily seen to equal $-\mathcal{A}_{\pi}(a,s)\mathcal{A}_{\pi}(a^\prime,s)$. Thus the expression reduces to:
\begin{align*}
    &\cov_{\tau\sim\pi}(\hat{\mathcal{A}^\prime}_{t,N}(a),\hat{\mathcal{A}^\prime}_{t,N}(a^\prime)|S_t)\\
    &=\sum_{k=1}^{N-1}\gamma^{2k}\cov_{\tau\sim\pi}\left(\frac{p_k(a|S_t,S_{t+k})}{\pi(a|S_t)}\delta_{t+k},\frac{p_j(a^\prime|S_t,S_{t+j})}{\pi(a^\prime|S_t)}\delta_{t+k}\middle|S_t\right)-\mathcal{A}_{\pi}(a,S_t)\mathcal{A}_{\pi}(a^\prime,S_t)\\
    &=\sum_{k=1}^{N-1}\gamma^{2k}\E_{\tau\sim\pi}\left[\frac{p_k(a|S_t,S_{t+k})}{\pi(a|S_t)}\frac{p_j(a^\prime|S_t,S_{t+j})}{\pi(a^\prime|S_t)}\delta_{t+k}^2\middle|S_t\right]-\mathcal{A}_{\pi}(a,S_t)\mathcal{A}_{\pi}(a^\prime,S_t).
\end{align*}
Note that every term in the sum is positive, hence the covariance of advantage estimates for distinct actions under the $\delta$HCA estimator is greater than or equal to the covariance under the MC estimator. Returning to the update variance in Equation \ref{update_variance}, we see that the covariance of the advantages appears in a product with the policy gradients of two different actions. While we can’t say in general that this product will always be negative, note that the sum of the policy gradient over all actions must be zero. Thus, speaking loosely, the product of policy gradients for two different actions will tend to be negative. This means that the greater covariance of the $\delta$HCA estimator will tend to decrease the variance of the update relative to the MC estimator, compounding the effect of the reduced variance. 

Note that when there are only 2 possible actions, the product of policy gradients in equation \ref{update_variance} will be strictly negative. Thus, in this special cases we can say that the variance of the update will always be lower or equal for the $\delta$HCA estimator compared to the MC estimator.

\section{Conclusion}
In this report, we introduce a new advantage estimator which is unbiased and has variance provably lower than or equal to that of the conventional Monte Carlo estimator when the necessary functions can be estimated exactly. To do so we build on the work of~\citet{HCA}. The primary distinction is that our estimator is based on TD-error instead of directly on rewards. This change is necessary to get the provided result. A limitation of this result is that it holds only when the exact value function $V_\pi$ and exact action probabilities $p_k(a|s,s^\prime)$ are available. This is particularly limiting because, in the case where $\hat{V}(s)=V_\pi(s)\ \forall s$, even the single step advantage estimator is unbiased, hence there would be no reason to add variance by using a larger $N$. For $N=1$, the $\delta$HCA and MC estimators are equivalent and thus $\delta$HCA confers no advantage. Nevertheless, given the form of our result, it seems reasonable to assume some level of robustness to errors in $\hat{V}(s)$ and $p_k(a|s,s^\prime)$. A potential direction for future work would be to explore how this estimator behaves when these functions are approximated with some nonzero error.

\newpage
\bibliographystyle{apacite}
\bibliography{refs.bib}
\end{document}